\documentclass{article}

\usepackage{fullpage}
\usepackage[pass]{geometry}
\usepackage{hyperref}       

\usepackage{burt}
\usepackage{graphicx}
\usepackage{enumitem}

\usepackage{url}            
\usepackage{booktabs}       
\usepackage{amsfonts}       
\usepackage{nicefrac}       
\usepackage{microtype}      
\usepackage{natbib}
\Crefname{prop}{Proposition}{Propositions}
\usepackage{bbm}

\begin{document}
\setcitestyle{authoryear}
\title{Variational Orthogonal Features}

\author{%
\begin{tabular}{c c c}
    David R. Burt & Carl Edward Rasmussen &Mark van der Wilk\\
     University of Cambridge& University of Cambridge &Imperial College London\\
     \texttt{drb62@cam.ac.uk} & \texttt{cer54@cam.ac.uk} &\texttt{m.vdwilk@imperial.ac.uk}
\end{tabular}
}
\maketitle

\begin{abstract}
\noindent Sparse stochastic variational inference allows Gaussian process models to be applied to large datasets. The per iteration computational cost of inference with this method is $\mathcal{O}(\tilde{N}M^2+M^3),$ where $\tilde{N}$ is the number of points in a minibatch and $M$ is the number of `inducing features', which determine the expressiveness of the variational family. Several recent works have shown that for certain priors, features can be defined that remove the $\mathcal{O}(M^3)$ cost of computing a minibatch estimate of an evidence lower bound (ELBO). This represents a significant computational savings when $M\gg \tilde{N}$. We present a construction of features for any stationary prior kernel that allow for computation of an unbiased estimator to the ELBO using $T$ Monte Carlo samples in $\mathcal{O}(\tilde{N}T+M^2T)$ and in $\mathcal{O}(\tilde{N}T+MT)$ with an additional approximation. We analyze the impact of this additional approximation on inference quality.
\end{abstract}

\section{Introduction}
Gaussian processes (GPs) are commonly used as priors over functions in Bayesian non-parametric models. The posteriors of these models are expressive and reflect uncertainty in regions with little data. In the case of regression with Gaussian likelihood, the log marginal likelihood (LML) of these models can be computed analytically in $\mcO(N^3),$ with memory $\mcO(N^2),$ where $N$ is the number of training points. Inference is often performed by maximizing the LML with respect to model hyperparameters (i.e.~empirical Bayes). For applications involving large datasets, exact inference is infeasible due to the high memory and computational burden. `Sparse' methods, which summarize the posterior process using a small set of features can be used to improve scalability. Sparse methods can be formulated as a  variational inference problem \citep{titsias2009variational}, in which the goal is to find the sparse approximation closest to the full posterior as measured by the Kullback-Leibler (KL) divergence. \citet{burt2019rates} show that under reasonable assumptions, very sparse representations of the posterior can still lead to accurate approximations.

For particularly large datasets it is desirable to perform inference without needing a complete pass through the dataset for each hyperparameter update. \citet{hensman2013gaussian} proposed using stochastic variational inference (SVI) in sparse GP models, which allows for each iteration of inference to be performed in $\mcO(\tilde{N}M^2+M^3)$ with memory complexity $\mcO(\tilde{N}M+M^2),$ where $M$ is the number of \emph{inducing features} that determine the variational family, and $\tilde{N}$ is the size of a minibatch.  \citet{hensman2016variational} showed that the time complexity of SVI can be reduced to $\mcO(\tilde{N}M^2)$ for GPs with Mat\'{e}rn covariance functions with half-integer shape parameter, by choosing a set of features that lead to structured covariance matrices. These features have a diagonal plus low-rank feature covariance matrix, $\Kuu$. In this work, we construct a large family of features that lead to a diagonal $\Kuu$ and are applicable to inference with \emph{any} prior with stationary kernel.

In \cref{sec:background}, we review SVI in GP models, as well as several results our method relies on. In \cref{sec:mainVOF}, we construct a new family of features, \emph{variational orthogonal features} (VOF), and describe their properties. In cases where the ELBO can be evaluated analytically, VOF have per iteration computational complexity $O(\tilde{N}M^2)$. With an additional mean-field approximation that we show in some cases can be made without loss of approximation quality, this complexity can be reduced to $O(\tilde{N}M)$. Alternatively, or in cases when the ELBO cannot be evaluated analytically, we can use Monte Carlo (MC) methods to obtain an unbiased estimator of the ELBO in $\mcO(\tilde{N}T+M^2T)$, with $T$ being the number of samples used for the MC estimate. The mean field approximation can also be applied in this context leading to a per iteration complexity of $\mcO(\tilde{N}T+MT)$.

\section{Background}\label{sec:background}
Throughout this work, we consider the problem of inference in a Bayesian model with data $\mcD=\{\bfx_n,y_n\}_{n=1}^N,$ $\bfx_n \in \R^D, y_n \in \R.$ We take a zero mean GP prior over mappings from $\R^D \to \R$, with covariance function $k: \R^D \times \R^D \to \R$ and a factorized likelihood, i.e.,
\begin{equation}\label{eqn:model}
f \sim \GP(0, k(\cdot,\cdot)) \text{, \,\,} p(\bfy|f(\bfX))=\prod_{n=1}^N p(y_n \vert f(\bfx_n)),
\end{equation}
where $\bfX = (\bfx_n)_{n=1}^N$, $\bfy= (y_n)_{n=1}^N$, $f(\bfX)=(f(\bfx_n))_{n=1}^N$ and $p(\cdot | f(\bfx_n))$ is the specified by the choice of likelihood.
\paragraph{Sparse Variational Inference in GP Models}
Variational GP methods define an approximate posterior process and minimize the KL divergence between this approximation and the full process. Following earlier sparse GP methods \citep{seeger-fast03, snelson2006sparse}, \citet{titsias2009variational} proposed choosing a set of inducing points $Z=\{\bfz_m\}_{m=1}^M,$ along with corresponding values of the process at these points $\{u_m=f(\bfz_m)\}_{m=1}^M.$ A Gaussian distribution is placed over these points, $q(\bfu) \sim \mcN(\bfm,\bfS),$ and the approximate posterior process is a Gaussian process with mean and covariance functions given by,
\begin{equation}
 \mu_q(\bfx_*) = \bfk_{\bfx_*\bfu}\Kuu^{-1}\bfm \text{\quad and \quad} k_q(\bfx_*,\bfx_*') = k(\bfx_*,\bfx_*') + \bfk_{\bfx_*\bfu}\Kuu^{-1}\left(\bfS - \Kuu\right)\Kuu^{-1}\bfk_{\bfu\bfx_*'} \label{eq:qf}
\end{equation}
where $[\Kuu]_{m,m'}=\cov(u_m,u_{m'})$ and $[\bfk_{\bfu \bfx_*}]_{m}=\cov(u_m, f(\bfx_*)).$ \citet{titsias2009variational} analytically found the optimal variational distribution for a given $Z.$ \citet{hensman2013gaussian, hensman2015scalable} proposed treating $\bfm$ and $\bfS$ explicitly as variational parameters, allowing for minibatches to be used during optimization, as well as inference with non-conjugate likelihoods. This approach gives the evidence lower bound (ELBO):
\begin{align}\label{eqn:ELBO}
\mcL \coloneqq  \sum_{n=1}^N \Exp{q(f_n)}{\log p(y_n \vert f_n )}-\KL{q(\bfu)}{p(\bfu)} \leq \log p(\bfy).
\end{align}
where $f_n=f(\bfx_n)$. Each term in the sum on the RHS can be computed analytically in the case of a Gaussian likelihood, and estimated to high precision via Gauss-Hermite quadrature in the case of a general factorized likelihood. Further, the sum can be approximated via subsampling minibatches of data. The second term on the RHS is analytic. The computational cost of computing an unbiased estimate of \cref{eqn:ELBO} is dominated by calculating the marginal mean and variance of $q(\bff)$ (\cref{eq:qf}) in order to estimate the first term. This is commonly implemented with a Cholesky decomposition of $\Kuu$, which is $\mcO(M^3)$, and a back-solving operation, which is $\mcO(\tilde{N}M^2)$. Given a Cholesky factor of $\Kuu$ the KL-term can be computed in $\mcO(M^2)$.

\paragraph{Interdomain Inducing Features}
Interdomain inducing features, \citep{NIPS2009_3876} generalize the notion of inducing points to linear transformations of the original process. For some collection of integrable functions $\{g_m\}_{m=1}^M$, define $u_m = \int_{\R^D} g_m(\bfx) f(\bfx) d\bfx.$ As before, we form a variational posterior of the form given in \cref{eq:qf}.

In the special case of Mat\'{e}rn half-integer kernels, \citet{hensman2016variational} defined \emph{Variational Fourier Features} (VFF). These features are defined in such a way that $\cov(u_m,f(x))= \cos(mx)$ or $\sin(mx)$, independent of the kernel hyperparameters. In the case of inference with a Gaussian likelihood using VFF, after an initial cost of $\mcO(NM^2)$ to accumulate statistics of the data, each iteration of hyperparameter optimization can be performed in $\mcO(M^3)$. In the case of other likelihoods, or when minibactching is used, VFF results in a low-rank plus diagonal structure for $\Kuu$, which makes the per-iteration cost of SVI $\mcO(\tilde{N}M^2)$ as opposed to the $\mcO(\tilde{N}M^2+M^3)$ per iteration cost with inducing points. It is this latter computational savings we seek to make more generally applicable in this work.

\paragraph{Monte Carlo Estimation of the ELBO}\label{sec:MCELBO}
In this paper, we will define features in such a way that we can evaluate the covariance matrix $\Kuu$, but only have access to Monte Carlo estimates of the cross covariance matrix $\Kuf$. Unbiased estimators of $\Kuf$ can be used to obtain an unbiased estimate of the ELBO (\cref{eqn:ELBO}) in the case of conjugate GP regression \citep{vanderWilk2018}. In the case where the likelihood is Gaussian with variance $\sigma^2$, \cref{eqn:ELBO} becomes
\begin{align*}
    \mcL = \sum_{n=1}^N&\left(-\log(2\pi \sigma)^2 - \frac{1}{2\sigma^2}(y_n^2-2y_n \mu_n + \mu_n^2 +\sigma_n^2)\right) - \KL{q(\bfu)}{p(\bfu)},
\end{align*}

where $\mu_n = \bfK_{\bfu, \mathrm{f_n}}\transpose\Kuu^{-1}\bfm$ and $\sigma_n^2 = k(\bfx_n,\bfx_n)+ \bfK_{\bfu, \mathrm{f_n}}\transpose\Kuu^{-1}(\bfS-\Kuu)\Kuu^{-1}\bfK_{\bfu, \mathrm{f_n}}$.
\Citet{vanderWilk2018} note that unbiased estimators for $\mu_n,$ $\mu^2_n$ and $\sigma^2_n$ are sufficient for performing inference. We employ this approach in this paper. Using P\'olya-Gamma variables \citep{polson2013bayesian}, this approach can be extended to the classification setting. However, we focus on conjugate regression to illustrate the ideas in this paper.   
\section{Variational Orthogonal Features}\label{sec:mainVOF}
We would like to obtain some of the computational benefits of VFF applied to stochastic variational inference (\cref{eqn:ELBO}) for a larger class of kernels. The first computational bottleneck we consider is inverting $\Kuu.$ To solve this, we construct features that can be applied to \emph{any} stationary kernel, so that $\Kuu$ is diagonal. Consider the entries of $\Kuu$ for interdomain features defined by $u_m = \int_{\R^D} g_m(\bfx) f(\bfx) d\bfx$,
 \begin{align}
     [\Kuu]_{m,m'}  = \cov(u_m, u_{m'}) \nonumber&= \Exp{}{\int_{\R^D} g_m(\bfx)f(\bfx) d\bfx\int_{\R^D} \overline{g_{m'}(\bfx')}f(\bfx') d\bfx'} \nonumber \\
     &= \int_{\R^D}g_m(\bfx)\int_{\R^D} \overline{g_{m'}(\bfx')}k(\bfx,\bfx') d\bfx  d\bfx'. \label{eqn:covuintermediate}
\end{align}
This nearly factors into two separate integrals; the only term depending on both $\bfx$ and $\bfx'$ is $k(\bfx,\bfx').$
Bochner's theorem (e.g.\citet[Theorem 4.1]{Rasmussen2005gpml}) tells us that for a stationary kernel there exists a non-negative integrable function $s: \R^D \to [0,\infty)$, the \emph{spectral density} of $k$, such that
\begin{equation}
    k(\bfx,\bfx') \eqqcolon \kappa(\bfx-\bfx') = (2\pi)^{-D/2}\int_{\R^D}\! e^{-i \omega \cdot (\bfx-\bfx')}s(\omega) d\omega. \label{eqn:bochner}
\end{equation}

\noindent This theorem motivates several spectral approximations to the covariance matrix, notably Random Fourier Features \citep{rahimi2008random}. We apply \cref{eqn:bochner} to expand $k(\bfx,\bfx')$ in \cref{eqn:covuintermediate}, so that it separates as:
\begin{align}
     \!\cov(u_m, u_n) \!=\!(2\pi)^{\frac{D}{2}}\int_{\R^D}\int_{\R^D}\!g_m(\bfx)e^{-i\omega \cdot \bfx}\frac{d\bfx}{{(2\pi)^{\frac{D}{2}}}} \overline{\int_{\R^D}g_{m'}(\bfx)e^{-i\omega \cdot \bfx'}\frac{d\bfx'}{{(2\pi)^{\frac{D}{2}}}}}s(\omega)d\omega. \label{eqn:covu}
\end{align}
The inner integrals are $\mcF[g_m](\omega)$ and $\mcF[g_{m'}](\omega)$, where $\mcF$ denotes the Fourier transform. The outer integral is an inner product between these transforms, over the space $L^2(\R^D; (2\pi)^{D/2}s)$ (i.e. $L^2$ equipped with a measure with density proportional to $s$). \Cref{eqn:covu} has analogues in the RKHS literature, and can be seen as using an isometry from the RKHS with kernel $k$ into $L^2$, see \citet[Theorem 10.12]{wendland_2004}.

By applying Fourier inversion in \cref{eqn:covu}, we can translate an orthogonal basis of functions in $L^2(\R^D)$ into a set of orthogonal features. In particular, we consider a collection of square-integrable functions $\{\psi_m\}_{m=1}^M$ that are  pairwise orthogonal. Subject to decay and regularity conditions on $\psi_m/\sqrt{s}$, outlined in \Cref{app:conditions}:
\begin{prop}\label{prop:ortho}
Let $f$ be a zero-mean Gaussian process indexed by $\R^D$ with a stationary kernel with spectral density $s$. Consider a set of features defined $\{u_m\}_{m=1}^M$, with $u_m= \int g_m(\bfx) f(\bfx)d\bfx$, with $g_m(\bfx) =\mathcal{F}^{-1}[\psi_m/\sqrt{s}](\bfx)$, with $\{\psi_m\}_{m=1}^M$ as above. Then for $1 \leq m, m' \leq N$ $\cov(u_m, u_{m'})= c_m \delta_{m,m'}$ for some constants $\{c_m\}_{m=1}^M$.
\end{prop}

\noindent We can also consider the converse problem. Namely, do there exist orthogonal features of the form $u_m = \int g_m(\bfx) f(\bfx)d\bfx$ that are not of the form above?

\begin{prop}\label{prop:converse}
Assume $u_m = \int g_m(\bfx) f(\bfx)d\bfx$, with $f$ a zero-mean Gaussian process indexed by $\R^D$, with stationary kernel with spectral density $s$. Then if $g_m$ decays is smooth and rapidly decaying, it can be written in the form $g_m(\bfx) =\mathcal{F}^{-1}[\psi_m/\sqrt{s}](\bfx)$ for some $\{\psi_m\}_{m=1}^M$, $\psi_m: \R^D \to \R$ pairwise orthogonal in $L^2(\R^D)$.
\end{prop}
The precise conditions for \Cref{prop:ortho,prop:converse} are in \Cref{app:conditions}. For a given stationary kernel, \emph{Variational orthogonal features} (VOF) are any set of inducing features following the construction in \cref{prop:ortho}. 

To perform inference, we also need the entries of $\Kuf$:
 \begin{align}
     [\Kuf]_{m,n} &= \cov(u_m, f(\bfx')) = \Exp{}{\int_{\R^D}g_m(\bfx)f(\bfx) d\bfx f(\bfx')}= \int_{\R^D} g_m(\bfx)k(\bfx,\bfx') d \bfx \nonumber \\
          &= \int_{\R^D}\left(\int_{\R^D}g_m(\bfx)e^{-i\omega \cdot \bfx}\frac{d\bfx}{(2\pi)^{D/2}}\right) e^{i\omega \bfx'} s(\omega)d\omega =(2\pi)^{D/2}\mcF^{-1}[\mcF[g_m]s](\bfx'). \label{eqn:covuf}
 \end{align}

\noindent Substituting the definition of VOF into \cref{eqn:covuf}, 
\begin{equation}\label{eqn:covufvof}
\cov(u_m, f(\bfx')) = (2\pi)^{D/2} \mcF^{-1}[\psi_m \sqrt{s}](\bfx').
\end{equation}
In cases when \cref{eqn:covufvof} is not analytic, it can be evaluated via Monte Carlo (MC) integration which, as discussed in \cref{sec:MCELBO}, is sufficent to obtain unbiased estimators of the ELBO.
\subsection{Examples}
We now consider several realizations of the features described in \cref{prop:ortho}. The construction given in \cref{sec:hermite} is the only case in which we compute $\Kuf$ in closed form for the SE-kernel. This calculation leads to features that are a special case of the ``eigenfunction features'' described in \citet{burt2019rates}. While we believe this connection is interesting, the calculation in \cref{sec:hermite} is somewhat tedious and the details can be safely skipped. The other examples discussed in this section can be applied to general stationary kernels, but require MC estimation of the ELBO.

\subsubsection{Analytic Example: Hermite Functions Features}\label{sec:hermite}

We consider an example of variational orthogonal features in which $\mcF^{-1}[\psi_m \sqrt{s}](\bfx')$ can be computed in closed form. Suppose $D=1$ and choose $\psi_{m}(x)$ to be the normalized Hermite function defined by
\[
\psi_m^{(r)}(x)=(-i)^{-m}H_m(rx) e^{\frac{-r^2x^2}{2}} \frac{2^{-m/2}\sqrt{r}}{\pi^{1/4} \sqrt{m!}},
\]
where $H_m(x)$ is the Hermite polynomial of degree $m$ defined by $H_0(x)=1$, $H_1(x)=2x, H_{m+1}(x) = 2xH_m(x)-2mH_{m-1}(x)$ and $r$ is a variational parameter, and we have used $x$ instead of $\bfx$ to emphasize the assumption $d=1$. It follows from \citet[7.374]{gradshteyn2014table} that these functions are orthonormal in $L^2(\R).$

Suppose inference is being performed with a squared exponential (SE) kernel with lengthscale $\ell^2$ and variance $v.$ This kernel has spectral measure
\[
s(\omega) = v \ell \exp\left(\frac{ -\ell^2\omega^2}{2}\right).
\]

\noindent In \Cref{app:hermite} we show the corresponding VOF are:
\begin{align*}
u_m &=c_m\left(r^2-\frac{1}{2}\ell^2\right)^{-1/2} \int_{-\infty}^\infty f(x)\exp\left(-\frac{x^2}{2\left(r^2-\frac{1}{2}\ell^2\right)}\right) G_m(x;r,\ell) dx \nonumber
\end{align*}
with 
\begin{align*}
&c_m =  \left(\frac{2^\frac{-m-1/2}{2}\sqrt{r}}{\sqrt{m!\pi v \ell}}  \right)\left(\frac{r^2+\frac{1}{2}\ell^2}{r^2-\frac{1}{2}\ell^2}\right)^\frac{m}{2} \text{\, and \, \,} G_m(x;r, \ell)=H_m\left(\frac{rx}{\sqrt{\left(r^2-\frac{1}{2}\ell^2\right)\left(r^2+\frac{1}{2}\ell^2\right)}}\right).
\end{align*}

\noindent In order for $g_m(\omega)=\mcF[\psi_m^{(r)}/\sqrt{s}](\omega)$ to be well-defined we enforce $2r^2>\ell^2.$ The entries of $\Kuf$ are
\begin{align*}
 \cov&(u_{m},f(x_n)) \!= \!\sqrt{2\pi}v \ell c_m\left(r^2+\frac{1}{2}\ell^2\right)^{-1/2}  \exp\left(-\frac{x_n^2}{2\left(r^2+\frac{1}{2}\ell^2\right)}\right)\!\left(\frac{r^2-\frac{1}{2}\ell^2}{r^2+\frac{1}{2}\ell^2}\right)^{m}\!G_m(x_n;r, \ell).
\end{align*}

\noindent We show in \Cref{app:hermite} that this construction corresponds to \emph{eigenfunction inducing features} \citep{burt2019rates}, for the SE-kernel defined with respect to an input distribution $ \mcN\left(0, (4r^4-\ell^4)/(4\ell^2)\right)$. Eigenfunction features are orthogonal features defined by $u_m =\lambda_m^{-1/2}\int \phi_m(x) f(\bfx)p(\bfx)d\bfx,$ where $p(\bfx)$ is the density of a distribution posited on the inputs and $\phi_m(x)$ are the eigenfunctions of a \emph{kernel operator,} $\mcK: \mcK h(\bfx') = \int h(\bfx)k(\bfx,\bfx')p(\bfx)dx.$ For eigenfunction features, $\cov(u_m, u_n) =  \delta_{m,n}$ and $\cov(u_m, f(\bfx))= \lambda_m^{-1/2}\phi_m(\bfx)$,
where $\lambda_m$ is the eigenvalue of $\mcK$ corresponding to $\phi_m$. 

In most cases neither the eigenfunctions nor the eigenvalues can be computed in closed form. However, \cref{prop:converse} implies that all eigenfunction features for continuous, stationary kernels associated to distributions with sufficiently smooth and decaying densities are a special case of VOF. The choice of input measure implicitly defines an orthogonal set of function in $L^2(\R^D)$.


\subsubsection{Trigonometric Variational Orthogonal Features}\label{sec:trigvof}

An alternative to the Hermite functions construction given in \Cref{sec:hermite}, is to choose 
\[
\psi_{2m}(\omega) = \cos\left(\pi m\omega/(2a)\right)\mathbf{1}_{[-a,a]} \quad \text{and} \quad \psi_{2m+1}(\omega) =\sin\left((m-1)\pi\omega/(2a)\right)\mathbf{1}_{[-a,a]},\]
where $a \in (0,\infty)$ is a variational parameter. We refer to these features as TrigVOF. The matrix $\Kuf$ cannot be calculated in closed form for TrigVOF and we estimate the marginal likelihood via Monte Carlo methods. While the $\psi_m$ are not smooth, we give a heuristic justification that the resulting features should be well-defined and orthogonal in \cref{app:conditions}.
\subsubsection{Orthogonal Polynomials}
Many other collections of VOF can be defined through different choices of $\{\psi_m\}_{m=1}^M$. For example, VOF can be formed by modifying orthogonal polynomials by multiplying through by the square root of the weight function with respect to which they are orthogonal (the Hermite polynomials have a Gaussian weight function leading to the earlier construction). MC estimation will generally be necessary to compute estimates of the ELBO, as with TrigVOF as $\Kuf$ does not typically have a closed-form. 

\subsection{When a Factorized $q(\bfu)$ is (almost) Exact}\label{sec:meanfield}

To compute \cref{eqn:ELBO}, we need $\mu_n=\bfK_{\bfu, \mathrm{f_n}}\transpose\bfm$ and $\sigma_n^2 = k(\bfx_n,\bfx_n)+ \bfK_{\bfu, \mathrm{f_n}}(\bfS-\bfI)\bfK_{\bfu, \mathrm{f_n}}\transpose.$ In the case of VOF and assuming, without loss of generality, that $\|\psi_m(\omega)\|^2_2=1$, we have $\Kuu = \bfI.$ If $\bfS$ is diagonal, $\sigma_n^2$  can be computed in $\mcO(\tilde{N}M)$ instead of $\mcO(\tilde{N}M^2).$ The optimal $\bfS$ \citep{titsias2009variational} for conjugate regression with likelihood variance $\sigma^2$ in the case $\Kuu=\bfI$ is
\begin{align}
 \bfS^{*}&=\Kuu(\Kuu+ \sigma^{-2}\Kuf\Kuf\transpose)^{-1}\Kuu \nonumber =(\bfI+ \sigma^{-2} \Kuf\Kuf\transpose)^{-1},\label{eqn:optcov}
\end{align}
which is diagonal if $\Kuf\Kuf\transpose$ is diagonal. The matrix $\Kuf$ depends on the distribution of the $\bfx_n$.  
\begin{prop}\label{prop:meanfield}
Suppose that we are performing inference in a GP regression model with eigenfunction features defined with respect to input density $p(\bfx)$. Suppose that the training data is independently and identically distributed according to a distribution with density $p(\bfx)$. Then, for fixed $M$, as $N \to \infty,$ $N\bfS^*$ tends to a diagonal matrix with probability $1$.  
\end{prop}
\noindent As noted in \cref{sec:hermite}, for `nice' kernels, eigenfunction features are a special case of VOF, such at least for certain VOF such as the Hermite features discussed in \cref{sec:hermite}, \cref{prop:meanfield} is applicable.

\begin{proof}[Sketch of Proof]
For eigenfunction features, using $\cov(u_m, u_n) =  \delta_{m,n}$ and $\cov(u_m, f(x))= \lambda_m^{-1/2}\phi_m(x)$,
\begin{align*}
    \!\frac{1}{N\sigma^2}[\Kuf\Kuf\transpose]_{m,m'} \!=\!\frac{1}{N\sigma^2\sqrt{\lambda_m\lambda_{m'}}}\sum_{n=1}^N\!\phi_{m}(\bfx_n)\phi_{m'}(\bfx_n).
\end{align*}
For large $N,$ applying the strong law of large numbers on the right hand side,
\begin{align*}
     \lim_{N \to \infty}\frac{1}{N\sigma^2}[(\Kuf\Kuf\transpose)_N]_{m,m'} \asto \frac{1}{\sigma^2\sqrt{\lambda_m\lambda_{m'}}}\int \phi_{m}(\bfx)\phi_{m'}(\bfx)p(\bfx)d\bfx 
    = \delta_{m,m'}\lambda_m^{-1}\sigma^{-2}.
\end{align*}
 Using \cref{eqn:optcov}, and defining $[\bfLam]_{m,m'}=\delta_{m,m'}\lambda_m,$ 
\begin{align}
    N\bfS^{*}_N = \left(\frac{1}{N}\bfI+\frac{1}{N\sigma^2}(\Kuf\Kuf\transpose)_N\right)^{-1} = (\sigma^{-2}\bfLam^{-1}+\mcE_{N})^{-1},
\end{align}
where $[\mcE_N]_{j,k}$ tends to zero as $N \to \infty$ In \Cref{app:meanfield} we show $(\sigma^{-2}\bfLam^{-1}+\mcE_{N})^{-1} \to \sigma^{2}\bfLam.$
\end{proof}


\subsection{Estimating the ELBO with Samples}\label{sec:elbosample}
At first, it appears that the computational cost of $\mcO(\tilde{N}MT)$ for computing an unbiased estimate of the ELBO using $T$ samples is the best we can hope for, as we need to compute $\tilde{N}M$ cross covariances between features and inducing points. Consider the mean and variance of the variational approximation $q(f(\bfx_n)):$  
\begin{align*}
\mu_n &= \sum_{m=1}^M \bfm_m\int \frac{\sqrt{s(\omega)}}{(2\pi)^\frac{D}{4}} e^{-i\omega \bfx_n} \psi_m(\omega) d\omega = \int \frac{\sqrt{s(\omega)} }{(2\pi)^\frac{D}{4}}e^{-i\omega \bfx_n}\sum_{m=1}^M \bfm_m \psi_m(\omega) d\omega,\\
\sigma_n^2 &= k(\bfx_n, \bfx_n) - \bfK_{\bfu, \mathrm{f_n}}\transpose \left(\bfS-\bfI\right) \bfK_{\bfu, \mathrm{f_n}} \\
&=(2\pi)^{-D/2}\int \sqrt{s(\omega)} e^{-i\omega \bfx_n} \int \sqrt{s(\omega')} e^{i\omega' \bfx_n} \psi(\omega) \left(\bfS-\bfI\right) \psi(\omega')\transpose d\omega d\omega'. 
\end{align*}
where $\psi(\omega)$ is a vector of the $M$ functions $\{\psi_m\}_{m=1}^M.$ After interchanging the order of integration and summation, we see that the part of the calculation dependent on the features needs to only be computed once per batch. The computation of a batch of $\widehat{\mu}$ can be performed in $\mcO(\tilde{N}T+MT)$ where $T$ is the number of samples used in order to estimate the mean. We use independent sets of samples to compute two estimates of the mean in order to compute $\widehat{\mu_n^2}.$ 

If $\bfS$ is diagonal, computing a batch estimate of $\widehat{\sigma}_n^2$ requires $\mcO(\tilde{N}T+MT)$ operations; if $\bfS$ is a dense matrix, the time complexity is $\mcO(\tilde{N}T+M^2T)$. At prediction time we can directly estimate \cref{eqn:covuf} via Gauss-Hermite quadrature. This provides biased, but deterministic estimators of the mean and variance, and ensures that the variance estimator is non-negative.

For the TrigVOF, we can sample $\omega$ uniformly on $[-a,a]$ in order to obtain an unbiased estimate of all terms. For a fixed sample $\omega$, the estimator only depends on the spectral density of the kernel at $s(\omega)$, and therefore can be seen as performing variational inference in a parametric featurized linear regression model. We are able to obtain unbiased estimators of the marginal likelihood by combining the predictions of many such models. This estimator is similar to the estimator developed concurrently in \citet{evans2020quadruply}, in that both methods rely on obtaining unbiased estimators of the mean and variance of the variational posterior. However, they begin with a high dimensional parametric model and calculations are performed largely in feature space, whereas our estimator is fully non-parametric.

\subsection{Convergence of VOF}\label{sec:convergenceVOF}
In the case of a Gaussian likelihood, and assuming we can solve the convex optimization problem of finding the optimal $q(\bfu)$, in order to show that the variational posterior converges to the true posterior as $M \to \infty,$ it suffices to show $\text{tr}(\Kff-\Kuf\transpose\Kuu^{-1}\Kuf) \to 0$ \citep{burt2019rates}\footnote{Minimizing $\text{tr}(\Kff-\Kuf\transpose\Kuu^{-1}\Kuf)$ has long been a focus of Gaussian process involving Nystr\"om approximations, e.g.~\cite{herbrich2003fast}, and the implication of convergence of the resulting variational approximation is implicit in \citet{titsias2009variational}, \citet{titsias_variational_2014}.}, that is it suffices to show that each diagonal element of $\Kuf\Kuu^{-1}\Kuf$ tends to the corresponding diagonal element in $\Kff$ . In \Cref{app:convergence} we show that if $\{\phi_m\}_{m=1}^\infty$ spans $L^2(\R^D)$ then as $M \to \infty$, the approximate posterior formed by VOF becomes exact.

When $\{\phi_m\}_{m=1}^\infty$ span $L^2([-a,a])$, such as with TrigVOF, the features can only represent low-frequency information, as they do not depend on the spectral density of the kernel outside of $[-a,a]$ (note the ELBO and the predictive mean are independent of the spectral density of the kernel outside $[-a,a]$ in this case. For $M$ sufficiently large, the variational lower bound will favor large values of the variational parameter $a$ and exact inference will again be recovered as $M \to \infty$ if $a$ is globally optimized. 
\section{Experiments}
In this section, we present some simple experiments showing the feasability of variational features, as well as the effect of the restriction of the variational parameter $\bfS$ to be diagonal. All experiments are implemented using the `inducing variable' framework \citep{van2020framework} within GPflow. \citep{GPflow2017}. We investigate the performance of both the Hermite inducing variables, implemented in closed form as discussed in \cref{sec:hermite}, and the TrigVOF discussed in \cref{sec:trigvof} which are implemented using Monte Carlo approximation as discussed in \cref{sec:elbosample}.
\subsection{Choice of Variational Parameters}
As discussed in \cref{sec:convergenceVOF}, even as $M$ tends to infinity, if the parameter $a$, is fixed, the TrigVOF will not recover exact inference. In \cref{fig:matrix_approx}, we show the quality of approximation of the matrix $\Qff$ to $\Kff$ for different choices of $a$ and $M$, If $a$ is small (bottom left), we only recover a band-limited version of the kernel. If $a$ is large and $M$ is not sufficiently large, the approximation is only accurate over a narrow part of input space (top left). 
\begin{figure}
    \centering
    \includegraphics[width=\textwidth]{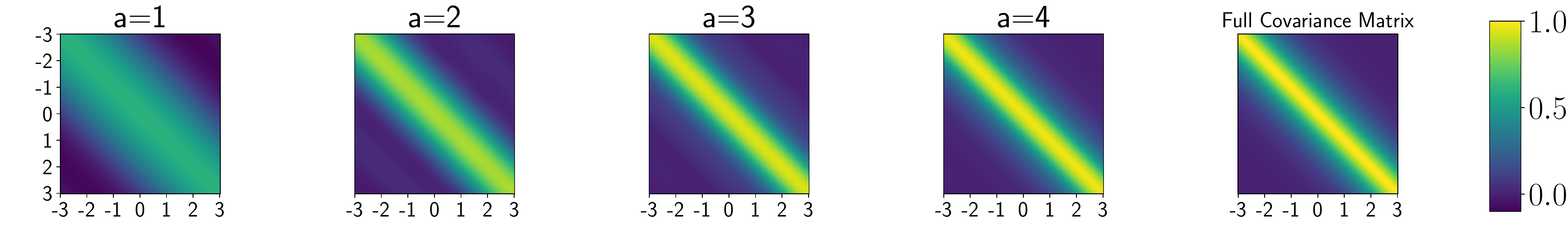}
    \includegraphics[width=\textwidth]{Increasing_a_m_31.pdf}
    \caption{Approximation of Mat\'{e}rn 3/2 kernel matrix with lengthscale $0.2$ by Trig VOF over $[-3,3].$ The top row uses $a=10,$ the bottom row uses $M=31.$}
    \label{fig:matrix_approx}
\end{figure}
The impact of mis-specifycing $a$ on the quality of inference is shown in \cref{fig:1d_regression}, for a synthetic, one-dimensional dataset. The top left plot is the result of choosing $a$ to be two small, so that only low-frequency featurs are modelled, while the bottom left is the result of choosing $a$ to small for the given $M$, so that the features cannot represent the data well over the entire domain. The top right hand figure is the result of optimizing $a$ using the variational lower bound, and retains most of the features of the exact posterior shown in the bottom right. Similar considerations arise when features defined with respect to a collection of orthogonal polynomials on a fixed interval are used.
\begin{figure*}
    \centering
    \includegraphics[width=\textwidth]{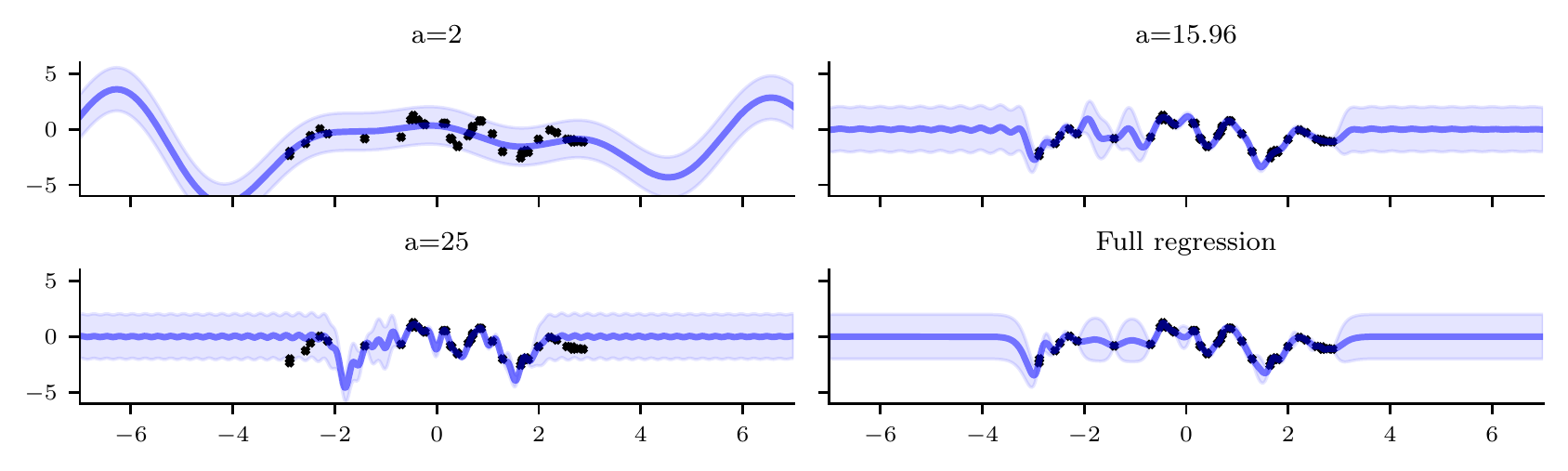}
    \caption{One dimensional regression with (fixed) Mat\'{e}rn 5/2 kernel, $N=80$ and $M=31.$ The first 3 models are trained with VOF with a trigonometric basis on $[-a,a]$. A dense covariance matrix $\bfS$ was used.}
    \label{fig:1d_regression}
\end{figure*}
\subsection{Effect of the Diagonal Approximation}
We consider the impact of the diagonal approximation to $\bfS$ on several synthetic datasets for both Hermite and Trig Features with a squared exponential kernel. From \cref{prop:meanfield}, if the inputs are Gaussian distributed, for $N$ sufficiently large we expect the diagonal approximation to be close to exact inference for the Hermite features.  \Cref{fig:Diagonal_Approx} shows ELBO plotted as a function of the number of features for both diagonal $\bfS$ and general $\bfS$. 
\begin{figure}[t]
    \centering
    \includegraphics[width=\textwidth]{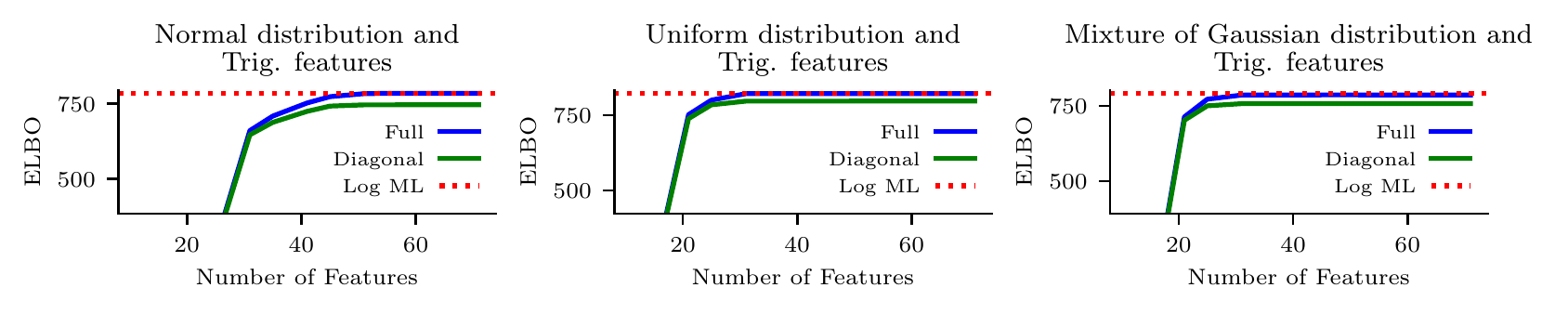}
    \includegraphics[width=\textwidth]{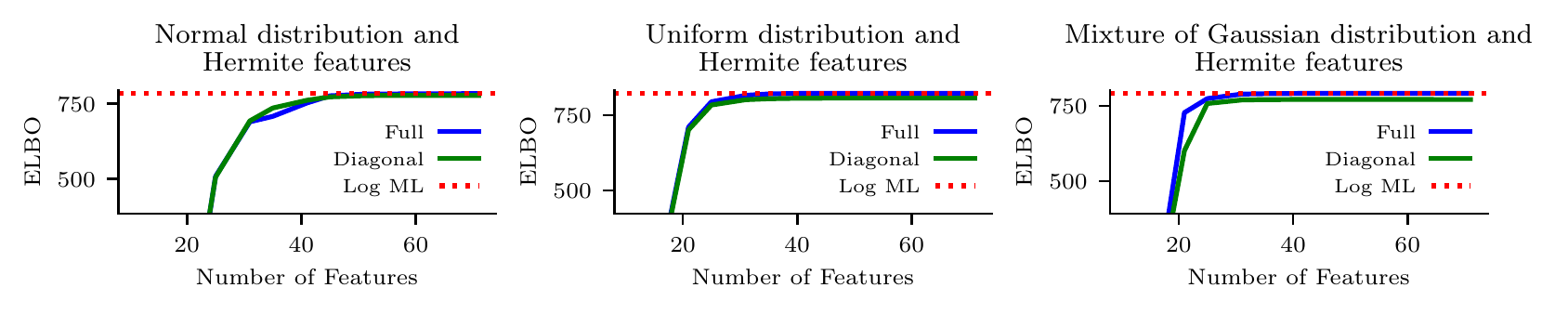}
    \caption{A comparison of dense covariance matrices (blue) and diagonal covariance matrices (green) for both the Trigonometric (top) and Hermite (bottom) VOF. The quality of the diagonal approximation is dependent on the covariate distribution. From left to right a Gaussian distribution, a uniform distribution and a mixture of Gaussians. In \cref{sec:meanfield} we showed that the optimal covariance matrix for the Hermite features and Gaussian distribution tends to a diagonal matrix as the data set becomes large. This explains the negligible difference in the cost of inference between the full rank and diagonal methods. (Note that the full matrix may occasionally perform worse than the diagonal matrix due to challenges with optimization). }
    \label{fig:Diagonal_Approx}
\end{figure}
The diagonal approximation has almost no effect on the quality of inference (as measured by the ELBO) for the Hermite features when the inputs are Gaussian or uniformly distributed. At times, a better lower bound is obtained using the diagonal approximation, which must be the result of difficulties with optimization when jointly optimizing variational parameters and model hyperparameters. When the input variance is multimodal, there is a more noticeable gap between the quality of the approximation obtained with a diagonal $\bfS$ using the Hermite features as opposed to with the full $\bfS$. There is generally a somewhat noticeable gap between the TrigVOF with the diagonal $\bfS$ and TrigVOF with the full $\bfS$.
\subsection{Limitations}
\paragraph{Scaling in dimension and additive models}
Like many other structured forms of GP approximations, VOF struggle with inputs that are in a high-dimensional space. Even if the training inputs lie on a smooth low-dimensional space, as VOF are defined without reference to the input distribution (unlike most implementations of inducing points) many more features will be needed for high-dimensional data. This can be avoided by placing additional structure on the prior, for example the additive structure considered in \citet{hensman2016variational} when using VFF.

\paragraph{Stochasticity in objective function}
Stochasticity in the evaluation of the evidence lower bound represents another significant obstacle to the practical application of this method. In particular, it is not clear how $T$ may need to scale with $M$ in order to achieve low-variance estimates of the ELBO. We found that variance in estimates of the ELBO represents a signficant obstacle.

\paragraph{Practical Implementation Obstacles}
While the per iteration computational cost of the Hermite features of $\mcO(\tilde{N}M^2)$, which is significantly smaller for $M \gg \tilde{N}$ than using the same number of inducing inputs, in practice the difference in computational time can be quite small. First, as the features are essentially limited to low-dimensional or additive models, the number of features needed for inference will often be quite small, so that $\mcO(M^3)$ may often be comparable to $\mcO(\tilde{N}M^2)$ for reasonable choices of minibatch size. Secondly, in the case of Hermite features evaluating $\Kuf$ requires evaluating a scaled version of the Hermite polynomials at the datapoints in each iteration. Both the evaluation of this quantity, as well as its derivative can be evaluated using standard recursion properties of the Hermite polynomials. However, due to the recursive nature of these formulas, in our implementation we find that the computational savings over inducing points is small for reasonable choices of $M$. Similar considerations arise when using Monte Carlo estimation for features defined with respect to other families of orthogonal polynomials.

\section{Related Work}

Several methods for Gaussian process models utilize similar approaches to the one discussed in this work. \citet{solin2014hilbert} use solutions to differential equations to construct an approximate series expansion to the kernel, leading to a parametric prior that resembles the Gaussian process prior, with each feature independent under the parameteric prior. Recently, \citet{evans2020quadruply} rely on beginning with a parametric model with features that are uncorrelated under the prior (e.g.~Random Fourier features) and perform variational inference in this model to improve scalability. Many of the MC estimates in their work closely resemble those employed here. 

Within the non-parametric variational framework, \citet{hensman2016variational} constructed VFF, which have a low-rank plus diagonal $\Kuu$, but are only applicable to Mat\'{e}rn kernels. \Citet{solin2019know} constructed \emph{variational harmonic features} based on approximately solving for harmonics of the Laplace operator. Variational harmonic features lead to a diagonal $\Kuu$ and can be applied to any stationary kernel, but the GP must be defined on a bounded subspace of $\R^D$, subject to boundary conditions. \citet{shi2019sparse} compute a matrix inverse to construct two sets inducing points that are independent from each other under the prior in order to improve scalability. \citet{burt2019rates} introduced \emph{eigenfunction features} which have a diagonal $\Kuu$ as a means of analyzing possible convergence rates for inducing point methods. They require analytic solutions to the integral equation for $\Kuf$. VOF generalize eigenfunction features for stationary kernels on $\R^D$, and are not limited by the need for closed-form solutions.

\section{Conclusions}
Varitional orthogonal features can be applied to Gaussian process conjugate regression tasks with stationary prior kernels and achieve better computational scaling in the number of features than existing methods. Constructing new VOF straightforward given any family of orthogonal functions. Methods for improving the scaling in data dimension, when using non-additive kernels are a promising direction for future work. This likely involves introducing some form of adaptivity, that allows the chosen basis to depend more strongly on the observed data. 

\section*{Acknowledgements}
Thanks to Nicolas Durrande for pointing to references regarding connections to the RKHS literature.
\bibliographystyle{abbrvnat}
\bibliography{preprint}

\onecolumn
\appendix

\section{Conditions for Proposition 1}\label{app:conditions}
We first prove that any collection of inducing features of the form $u_m(\bfx) = \int g_m(\bfx)f(\bfx)d\bfx$ with $\cov(u_m,u_n)=c_m\delta_{m,n}$ can be written in the form stated in \cref{prop:ortho}. We assume the mean function is $0$, \footnote{Bounded mean would suffice with minor modifications to the proof.} the kernel has a spectral measure with density with respect to Lebesgue measure and $g_m(\bfx) \in L^1(\R^D)$, is measurable and real valued for all $m$. As the mean function is $0$, 
\begin{align}\label{eqn:assumption}
c_m \delta_{m,n} &= \Exp{\!}{\int g_m(\bfx)f(\bfx)d\bfx\int g_n(\bfx')f(\bfx')d\bfx'}  =\Exp{\!}{\int_{\R^D \times \R^D} g_m(\bfx)g_n(\bfx')f(\bfx)f(\bfx')d(\bfx,\bfx')}
\end{align}
for all $m,n$. In order to justify the second equality, it suffices to show that the integral over the product measure converges absolutely (almost surely with respect to the Gaussian process).  By Markov's inequality, it suffices to show it converges absolutely in expectation, i.e.

\[\Exp{}{\int_{\R^D \times \R^D} |g_m(\bfx)g_n(\bfx')f(\bfx)f(\bfx')|d(\bfx,\bfx')}< \infty.\]

Again applying Fubini's theorem this is the case if
\[
\int_{\R^D\times\R^D} |g_m(\bfx)||g_n(\bfx')| \Exp{}{|f(\bfx)f(\bfx')|} d(\bfx,\bfx')< \infty.
\]
Using Cauchy-Schwarz and stationarity of the process, 
\[
\Exp{}{|f(\bfx)f(\bfx')|} \leq \sqrt{\Exp{}{|f(\bfx)|^2}\Exp{}{|f(\bfx')|^2}} = \Exp{}{|f(0)|^2} = C,
\]
where $C<\infty$ is the (uncentered) second moment of a half-normal distribution with variance $k(0,0)$. Then,
\[
\int_{\R^D\times\R^D} |g_m(\bfx)||g_n(\bfx')| \Exp{}{|f(x)f(x')|} d(\bfx,\bfx')\leq C\int_{\R^D\times\R^D} |g_m(\bfx)||g_n(\bfx')|d(\bfx,\bfx') < \infty.
\]
where the final inequality uses $g_m ,g_n \in L^1(\R^D)$. Hence \cref{eqn:assumption} holds. 

As the expectation of the absolute value of the integral converges, we may also interchange the expectation and integrals in \cref{eqn:assumption}, giving,
\[
c_m \delta_{m,n}=\int \int  g_m(\bfx)g_n(\bfx')\Exp{}{f(\bfx)f(\bfx')}d\bfx d\bfx'=\int \int  g_m(\bfx)g_n(\bfx')k(\bfx,\bfx')d\bfx d\bfx'.
\]
As our kernel is assumed stationary, we can apply Bochner's Theorem, using the assumption that the spectral measure has density, which we denote by $s$, to rewrite the RHS,
\[
c_m \delta_{m,n}=\int \int  g_m(\bfx)g_n(\bfx')\int e^{-i\omega\cdot \bfx}\overline{e^{-i\omega\cdot \bfx'}}s(\omega)d\omega d\bfx d\bfx'
\]
As each of the iterated integrals converges absolutely, we may again use Fubini's theorem,
\[
c_m \delta_{m,n}=\int \left(\int g_m(\bfx) e^{-i\omega\cdot \bfx}d\bfx\sqrt{s(\omega)}\right)\left(\overline{\int  g_n(\bfx')e^{-i\omega\cdot \bfx'}   d\bfx'\sqrt{s(\omega)}}\right)d\omega
\]
As $g_m \in L^1(\R^D)$ its Fourier transform is bounded, so $\mcF [g_m]\sqrt{s} \in L^2(\R^D)$.
We conclude, $\psi_m\coloneqq \mcF [g_m]\sqrt{s} \in L^2(\R^D)$ are pairwise orthogonal in $L^2(\R^D)$. Dividing both sides by $\sqrt{s}$ and applying Fourier inversion completes the proof of this direction.  

The proof of sufficiency follows by reversing the steps of the above argument, making the necessary assumptions on $\psi_m(\omega)/s(\omega)$  (integrability and integrability of Fourier transform) so that the necessary Fourier transforms exist and Fubini's theorem can be applied.

Hermite VOF with a squared exponential kernel lead to a $g_m$ that is both infinitely differentiable and integrable if $2r^2>\ell^2$. In particular, in this case $\psi_m/\sqrt{s}$ is a member of the \emph{Schwartz space} a class of functions that are rapidly decaying with rapidly decaying derivatives. As the Fourier transform maps Schwartz functions to Schwartz function, absolute integrability of the Fourier transform follows.

While we empirically find that using basis functions that are piece-wise continuous (e.g. TrigVOF) do not perform pathologically, there is more difficulty in rigorously verifying that they are well-defined. The inverse Fourier transform of a function that is piece-wise continuous but not continuous is not absolutely integrable (e.g. the Fourier transform of $\mathbf{1}_{[-a,a]}$ is a sinc function). We note that the covariance matrices only depend on $L^2$ properties of Fourier transform of $\psi_m/\sqrt{s}$. As these are persevered under the Fourier transform, approximating the piece-wise continuous functions by smooth functions in $L^2$ and taking Fourier transforms of these functions would lead to an inference scheme with well-defined features that is arbitrarily close to the inference scheme developed with the TrigVOF.
\section{Hermite Feature Derivation}\label{app:hermite}
The Hermite polynomials in one dimension are defined by,
\begin{equation}\label{eqn:recurrence}
H_0(x)=1, H_1(x)=x \text{ and } H_{n+1}(x)= 2x H_n(x)-2nH_{n-1}(x). 
\end{equation}
They satisfy the orthogonality relation:
\begin{equation}\label{eqn:hermiteortho}
 \int H_m(rx)H_n(rx) e^{-r^2x^2} dx = \frac{1}{r}\sqrt{\pi} 2^n n!\delta_{m,n}.
\end{equation}

 In order to compute the needed quantities we will use exponential generating function of $H_n(x).$ For all complex $t,r$ and $x,$ the following series expansion is valid \citep[8.957]{gradshteyn2014table}:
\begin{equation}\label{eqn:genFunc}
\exp\left(2rxt-t^2\right) = \sum_{n=0}^\infty \frac{t^n}{n!}H_n(rx) 
\end{equation}

\subsection{Fourier Transform Identity}
In order to compute the covariance matrix $\Kuf$ we will need to compute the Fourier transform of a Hermite function times $\exp(-\alpha x^2),$ for an arbitrary $\alpha>0.$ This computation essentially follows the same argument as the proof that Hermite functions are eigenfunctions of the Fourier transform, albeit with more bookkeeping. 
\begin{prop}\label{prop:hermiteft}
\begin{equation}
\mcF^{-1}[e^{-\alpha x^2} H_m(rx)] (\omega) =\sqrt{\frac{1}{2\alpha}} \exp\left(-\frac{1}{4\alpha}\omega^2\right)\left(i\sqrt{\frac{r^2-\alpha}{\alpha}}\right)^{n}H_n\left(\frac{r\omega}{\sqrt{4\alpha(r^2-\alpha)}}\right). 
\end{equation}

\end{prop} 

 \begin{proof}[Proof of Proposition]
 
 We begin with the generating function, \cref{eqn:genFunc}, multiplied by $e^{-\alpha x^2}$
 \begin{equation}\label{eqn:genfunceq}
\mcF^{-1}\left[\exp\left(-\alpha x^2+2rxt-t^2\right)\right] = \sum_{n=0}^\infty \frac{t^n}{n!}\mcF^{-1}\left[e^{-\alpha x^2}H_n(rx).\right]
 \end{equation}
The left hand side can be computed directly (via completing the square):
 \begin{align*}
\mcF^{-1}\left[\exp\left(-\alpha x^2+2rxt-t^2\right)\right](\omega) &= \frac{1}{\sqrt{2\pi}}\int_x  \exp\left(-\alpha x^2+2rxt-t^2\right)\exp(i\omega x) dx \\
&= \frac{1}{\sqrt{2\pi}}\exp(-t^2)\int_x  \exp\left(-\alpha x^2+(2rt+i\omega)x\right) dx \nonumber\\
&= \sqrt{\frac{1}{2\alpha}}\exp\left(-t^2 +  \frac{(2rt + i\omega)^2}{4\alpha}\right) \nonumber\\
&= \sqrt{\frac{1}{2\alpha}} \exp\left(-\frac{1}{4\alpha}\omega^2\right) \exp\left(\left(\frac{r^2-\alpha}{\alpha}\right)t^2 + \frac{ir\omega t}{\alpha} \right)\\
&= \sqrt{\frac{1}{2\alpha}} \exp\left(-\frac{1}{4\alpha}\omega^2\right) \exp\left(-t'^2 +\frac{r\omega t'}{\sqrt{\alpha(r^2-\alpha)}} \right).
\end{align*}
In the final line, we defined $t'= i\sqrt{\frac{r^2-\alpha}{\alpha}}t.$

Let $\omega'= \frac{\omega}{\sqrt{4\alpha(r^2-\alpha)}},$ then 
\begin{equation}\label{eqn:subomega}
    \mcF^{-1}\left[\exp\left(-\alpha x^2+2rxt-t^2\right)\right](\omega)= \sqrt{\frac{1}{2\alpha}} \exp\left(-\frac{1}{4\alpha}\omega^2\right) \exp\left(-t'^2 + 2r\omega't'\right). 
\end{equation}

\end{proof}
We now recall the left hand side of \cref{eqn:genfunceq} and rewrite the right hand side of \cref{eqn:subomega} using \cref{eqn:genFunc},
\begin{align}
\sum_{n=0}^\infty \frac{t^n}{n!}\mcF^{-1}\left[e^{-\alpha x^2}H_n(rx)\right]& = \sqrt{\frac{1}{2\alpha}} \exp\left(-\frac{1}{4\alpha}\omega^2\right) \sum_{n=0}^\infty \frac{t^{'n}}{n!}H_n(r\omega') \nonumber \\
&=\sqrt{\frac{1}{2\alpha}} \exp\left(-\frac{1}{4\alpha}\omega^2\right) \exp\left(-t'^2 - 2r\omega't'\right)\nonumber \\
 &= \sqrt{\frac{1}{2\alpha}} \exp\left(-\frac{\omega^2}{4\alpha}\right) \sum_{n=0}^\infty \frac{\left(i\sqrt{\frac{r^2-\alpha}{\alpha}}\right)^{n}t^n}{n!}H_n\left(\frac{r\omega}{\sqrt{4\alpha(r^2-\alpha)}}\right). \nonumber
\end{align}
By equating powers of $t:$
\[
\mcF^{-1}\left[e^{-\alpha x^2}H_n(rx)\right](\omega) = \sqrt{\frac{1}{2\alpha}} \exp\left(-\frac{1}{4\alpha}\omega^2\right)\left(i\sqrt{\frac{r^2-\alpha}{\alpha}}\right)^{n}H_n\left(\frac{r\omega}{\sqrt{4\alpha(r^2-\alpha)}}\right). 
\]
\subsection{Inducing Variables and Covariance}
 Given a SE-kernel with variance $v$ and lengthscale $\ell^2,$ i.e.
 \[
 k(x,x') = v\exp\left(-\frac{|x-x'|}{2\ell^2}\right).
 \]
 the corresponding spectral measure is:

\[
s(\omega)= v \ell \exp\left(\frac{ -\ell^2\omega^2}{2}\right).
\]

By \cref{eqn:hermiteortho}, $\{\psi_m\}_{m=1}^M=\left\{(-i)^{-m}H_m(rx) e^{-r^2x^2/2} \frac{2^{-m/2}\sqrt{r}}{\pi^{1/4} \sqrt{m!}}\right\}_{m=1}^M$ are orthonormal functions in $L^2(\mathbb{R})$ with Lebesgue measure (for $M=\infty$ they form a basis for $L^2(\mathbb{R})$).
Recall our inducing points are 
\begin{equation}
u_m := (2\pi)^{-1/4}\int \mcF^{-1}[\psi_m^{(r)}s^{-1/2}](x) f(x) dx.
\end{equation}
Then, 
\begin{align*}
    g_m(x)&= (2\pi)^{-1/4}\textbf{}\mcF^{-1}[\psi_m^{(r)}s^{-1/2}](x)\\ &= (2\pi)^{-1/4}\frac{2^{-m/2}\sqrt{r}}{\pi^{1/4}\sqrt{v \ell}  \sqrt{m!}}  \mcF^{-1}\left[H_m(r\omega) \exp\left(-\left(r^2-\frac{1}{2}\ell^2\right)\omega^2/2\right)\right](x). \nonumber
\end{align*}
Using \cref{prop:hermiteft} with $\alpha = \frac{1}{2}\left(r^2-\frac{1}{2}\ell^2\right),$ we have,
\begin{multline}
g_m(x) = (2\pi)^{-1/4}\frac{2^{-m/2}\sqrt{r}}{\pi^{1/4}\sqrt{v \ell}  \sqrt{m!}}  \sqrt{\frac{1}{\left(r^2-\frac{1}{2}\ell^2\right)}} \exp\left(-\frac{1}{2\left(r^2-\frac{1}{2}\ell^2\right)}x^2\right)\\ \times \left(\frac{r^2+\frac{1}{2}\ell^2}{r^2-\frac{1}{2}\ell^2}\right)^{m/2}H_m\left(\frac{rx}{\sqrt{\left(r^2-\frac{1}{2}\ell^2\right)\left(r^2+\frac{1}{2}\ell^2\right)}}\right). \label{eqn:hermitefeats}
\end{multline}
Combining \cref{eqn:covuf} and \cref{prop:hermiteft} with $\alpha = \frac{1}{2}\left(r^2+\frac{1}{2}\ell^2\right),$ we have,
\begin{multline*}
 \cov(u_{m},f(x)) = (2\pi)^{1/4}\frac{2^{-m/2}\sqrt{r\ell}}{\pi^{1/4}\sqrt{v}  \sqrt{m!}}  \sqrt{\frac{1}{\left(r^2+\frac{1}{2}\ell^2\right)}} \exp\left(-\frac{1}{2\left(r^2+\frac{1}{2}\ell^2\right)}x^2\right)\\ \times \left(\frac{r^2-\frac{1}{2}\ell^2}{r^2+\frac{1}{2}\ell^2}\right)^{m/2}H_m\left(\frac{rx}{\sqrt{\left(r^2-\frac{1}{2}\ell^2\right)\left(r^2+\frac{1}{2}\ell^2\right)}}\right). 
\end{multline*}
\subsection{Equivalence with Eigenfunction Features}
The eigenfunctions and eigenvalues of the SE-kernel with parameters $v, \ell^2$ defined with respect to input density $\mcN(0,\sigma^2)$ are given by:
\begin{align}
    \phi_m(x) = \exp(-(c-a)x^2)H_m(\sqrt{2c}x) \text{ \quad and \quad} \lambda_m= v\sqrt{2a/A}B^m 
\end{align}
with $a= 1/(4\sigma^2), b = 1/(2\ell^2), c = \sqrt{a^2+2ab}, A = a+b+c$ and $B=b/A.$
The corresponding eigenfunction inducing features, normalized so that $\Kuu= \bfI,$ are:
\begin{equation}
    u_m =\frac{1}{\sqrt{\lambda_m}} \int \phi_m(x)p(x)dx = \frac{1}{\sqrt{2\pi\sigma^2\lambda_m}}\int \exp(-(c-a)x^2)H_m(\sqrt{2c}x)\exp(-2ax^2) dx. \label{eqn:eigenfeats}
\end{equation}
Taking $\sigma^2 = \frac{4r^4-\ell^4}{4\ell^2}$ in the Hermite VOF with SE-Kernel, yields $a=\frac{\ell^2}{4r^4-\ell^4}, b =1/(2\ell^2),$ and $c=\frac{r^2}{2\left(r^2-\frac{1}{2}\ell^2\right)\left(r^2+\frac{1}{2}\ell^2\right)},$ and leads to \cref{eqn:hermitefeats} and \cref{eqn:eigenfeats} being equivalent.
\section{Proof for Proposition 2}\label{app:meanfield}
In the main text, we showed, 
\begin{align}
    N\bfS^{*}_N = N\left(\frac{1}{N}I+\frac{1}{N}(\Kuf\Kuf\transpose)_N\right)^{-1}= (\bfLam^{-1}+\mcE_{N})^{-1}
\end{align}
with $\mcE_N$ a matrix with entries that are are $o(1).$ Consider the matrix identity,
\[
(\bfLam^{-1}+\mcE_{N})^{-1} =\bfLam-\bfLam\mcE_{N}(\bfLam^{-1}+\mcE_{N})^{-1}.
\]
It suffices to show that all of the entries in
\[
\bfLam\mcE_{N}(\bfLam^{-1}+\mcE_{N})^{-1}
\]
tend to zero. The largest entry in any square matrix is bounded above by its largest operator norm. Recall that the operator norm is submultiplicative. As $\bfLam$ is a positive diagonal matrix, its operator norm is just the largest entry, equal to $\lambda_1$.

\[
\|\mcE_{N}\|_{op} \leq \|\mcE_{N}\|_{F} \leq M\max_{m\leq M}[\mcE_{N}]_{m,m},
\]
where $\| \cdot \|$ denotes the Frobenius norm, which is equal to the square root of the sum of the squared entries in a matrix. This tends to zero, as $M$ is fixed and all of the entries in $\mcE_{N}$ tend to zero as $N$ becomes large.

As $\bfLam^{-1}+\mcE_{N}$ is a symmetric positive definite matrix, the operator norm of $(\bfLam^{-1}+\mcE_{N})^{-1}$ is equal to its largest eigenvalue, which is the reciprocal of the absolute value of the smallest eigenvalue of $\bfLam^{-1}+\mcE_{N}$.
For any $\bfv \in \R^m,$
\[
\| (\bfLam+\mcE_{N})\bfv \|= \|\bfLam\bfv+\mcE_{N}\bfv\| \geq \lambda_{M}\|\bfv\| - \|\mcE_{N}\|_{op} \|\bfv\|
\]
where in the last line we used the reverse triangle inequality. We have already argued $\lim_{N\to \infty} \|\mcE_{N}\|_{op} =0$, so the largest eigenvalue of $(\bfLam+\mcE_{N})^{-1}$ tends to $\lambda_M^{-1}.$

 It follows that, $\|\bfLam \mcE_{N}(\bfLam^{-1}+\mcE_{N})^{-1}\|_{op}$ tends to $0,$ completing the proof of \Cref{prop:meanfield}.

\section{Convergence of Variational Orthogonal Features}\label{app:convergence}
In the case of regression, in order to show that the variational posterior converges to the true posterior as $M \to \infty,$ it suffices to show that $\text{tr}\left(\Kff-\Kuf\transpose\Kuu^{-1}\Kuf\right) \to 0$, where $\text{tr}(A)$ denotes the trace of $A$. Suppose that $\psi_m(\bfx)$ form a basis for $L^2(\mathbb{R}^d).$ As $k(\bfx,\bfx')$ is real, $s(\omega)$ is an even function, so we can write its Fourier transform as 
\[
\kappa(\bfx-\bfx') = (2\pi)^{-D/2}\int \cos(\omega \cdot (\bfx-\bfx'))s(\omega)d\omega.
\]
An arbitrary entry in $\Kff$ is given by,
  \begin{align*}
   k(\bfx,\bfx') &= (2\pi)^{-D/2}\int s(\xi) \cos(\xi\cdot(\bfx-\bfx'))d\xi \\
   &=(2\pi)^{-D/2} \int s(\xi) \cos(\xi\cdot\bfx)\cos(\xi\cdot \bfx')d\xi +(2\pi)^{-D/2}\int s(\xi) \sin(\xi\cdot \bfx)\sin(\xi\cdot \bfx')d\xi
  \end{align*}
  
We consider the first term, as the second can be handled in the same way. As we have chosen features that form a basis for $L^2(\R^D)$ and $\sqrt{s} \in L^2(\R^D)$
\[
 \sqrt{s(\omega)}\cos(\omega \cdot \bfx)= \sum_{m=1}^\infty a_{m,\bfx} \psi_m(\omega)
\]
  where $a_{m,\bfx} = \int \sqrt{s(\omega')}\cos(\omega \cdot \bfx) \psi_m(\omega') d \omega'$ (i.e. the projection of this function on to $\psi_m).$ 
 \begin{align*}
     \int \sqrt{s(\xi)} \cos(\xi\cdot\bfx)\sqrt{s(\xi)}\cos(\xi\cdot \bfx')d\xi & = (2\pi)^{-D/2}\sum_{m=1}^\infty\sum_{m'=1}^\infty a_{m,\bfx}a_{m',\bfx} \int \psi_m(\xi)\psi_{m'}(\xi)d\xi \\
     & = (2\pi)^{-D/2}\sum_{m=1}^\infty a_{m,\bfx}a_{m,\bfx'}\\
     & =(2\pi)^{-D/2}\sum_{m=1}^\infty \int \sqrt{s(\omega')}\cos(\omega' \cdot \bfx) \psi_m(\omega') d \omega' \\
     & \hspace{.5cm}\times \left(\int \sqrt{s(\omega)}\cos(\omega \cdot \bfx) \psi_m(\omega) d \omega\right).
 \end{align*}
 
Calculating directly,
\[
\Kuf\Kuf\transpose = (2\pi)^{-D/2}\sum_{m=1}^M \int \sqrt{s(\omega')}\exp(-i \omega \cdot \bfx) \psi_m(\omega') d \omega' \int \sqrt{s(\omega)}\exp(i \omega \cdot \bfx') \psi_m(\omega) d \omega.
\]
Assuming (for simplicity, though it is not essential to the argument) that all of the basis functions are purely even or odd functions, we see that the even basis functions converge to the integral involving cosines which we expanded above. The odd basis functions recover the integral involving sines.

The same argument shows that if our orthogonal functions are complete in some subspace of $S \subset L^2(\R^D),$ and zero outside, for example the Trig basis for $L^2([-a,a]),$ $\Kuf\Kuf\transpose$ will converge to kernel with Fourier transform $s(\omega)\mathbbm{1}_{S},$ where $\mathbbm{1}_{S}$ is the indicator function on $S$. 
\section{Experimental Details}\label{app:experiments}
\subsection{Implementation of sampling for Monte Carlo Estimation}
When sampling $\omega, \omega'$ for Trig VOF in order to perform \Cref{sec:elbosample}, we drew two independent sets of $K$ non-i.i.d uniform variables. In particular, we sampled a single random variable for each of the two sets, on $\omega_1,\omega'_1 \sim \mcU[0,1/T],$ and defined our samples $i\leq T$ as $\omega_i = (i-1)/T+\omega_1,\omega'_i = (i-1)/T+\omega'_1.$ This gives us to independent sets of variables, uniform on $[0,1]$ which were rescaled to $[-a,a]$ and used to obtain two estimators $\widehat{\mu_1}, \widehat{\mu_2}.$ We defined $\widehat{\mu}=(\widehat{\mu_1}+\widehat{\mu_2})/2$ and $\widehat{\mu^2}=(\widehat{\mu_1}\widehat{\mu_2}).$ We estimated the variance using the same samples, by taking an outer product of these samples (i.e. we used $2T$ samples in estimating the mean, but $T^2$ samples in estimating the variance). We found that placing samples on a grid led to a dramatic savings in sample efficiency and sharing samples between estimators allowed for several computations to be performed once instead of twice. 

\subsection{Investigation of Mean Field Inference (Figure 3)}
In order to show properties of the mean field approximation, 1000 training inputs were sampled according to either a $\mcN(0,3^2)$, $\mcU[-\sqrt{108},\sqrt{108}]$ or a mixture of two Gaussians one with variance 1 and the other with variance $0.5$, with the former having weight $0.7$ and the latter $0.3$. The means were set such that this distribution was mean centered and had standard deviation $3$.

The training outputs were generated by sampling a GP prior with zero mean and SE-kernel with variance and lengthscale $0.5$. Uncorrelated observation noise with standard deviation $.01$ was added to this sample. A standard Gaussian process regression model was fit using L-BFGS on the dataset in order to compute the full ML. 

Hermite VOF was parameterized in terms of the standard deviation of the input density associated to the corresponding eigenfunction features. The variational parameters, as well as kernel hyperparameters, were optimized with L-BFGS for $M \in \{11, 15, 21, 25, 31, 35, 41, 45, 51, 55, 61, 65, 71\}$.

The TrigVOF were trained using $100$ samples on the mean and $2500$ samples to estimate the variance. Training was performed with 30000 iterations of adam. The curves were made by averaging 5000 evaluations of the marginal likelihood, each computed over the full batch of 1000 point. This made the standard error of the estimate of the marginal likelihood negligible (generally $\leq .1$).

\subsection{One dimensional regression}
For the one dimensional regression example (Figure 2 in the main text) with TrigVOF, $N=80$ training inputs were drawn uniformly on $[-3,3].$ The training outputs were then sampled from a GP prior with Mat\'{e}rn 5/2 kernel lengthscale $0.2$ and variance $1$ and noise standard deviation $0.03.$ In order to show the impact of changing $a$ on the approximation of a given model the kernel and likelihoods were fixed for all models. In practice, if $a$ is fixed and the kernel is trainable, the model will favor overly smooth solutions, even if $a$ is fixed large. The full model was fit using standard GP regression. The trigonometric models used $M=31$ features, and were all trained using full batch stochastic variational inference to train $\bfm$ and $\bfS,$ with $\bfS$ diagonal (all hyperparameters were fixed).  Adam was run for 30000 iterations, with learning rate $.0005.$ $100$ samples were used to estimate the mean, and $2500$ samples were used to estimate the variance. The plot shows the mean function with $\pm$ 2 standard deviations shaded.

\end{document}